\documentclass[a4paper,12pt]{article}

\usepackage[utf8]{inputenc}
\usepackage[unicode=true,hypertexnames=false]{hyperref}
\usepackage{amsmath,amssymb,amsthm,amsfonts}
\usepackage{mathtools} 
\usepackage[linesnumbered,ruled]{algorithm2e}
\usepackage{xspace}
\usepackage[dvipsnames,table]{xcolor}
\usepackage{tikz}
\usepackage{multirow}

\usepackage{pgfplots}
\pgfplotsset{compat=newest}
\pgfplotscreateplotcyclelist{myplotcycle}{%
blue,mark=*,\\     
green!60!black,mark=triangle*,\\    
red!80!black,mark=square*,\\    
}

\newcommand{\oea}{\mbox{$(1 + 1)$~EA}\xspace}

\newcommand{\mplea}{\mbox{$(\mu+\lambda)$~EA}\xspace}
\newcommand{\mclea}{\mbox{$(\mu,\lambda)$~EA}\xspace}

\newcommand{\ollea}{${(1 + (\lambda , \lambda))}$~GA\xspace}

\newcommand{\onemax}{\textsc{OneMax}\xspace}
\newcommand{\leadingones}{\textsc{LeadingOnes}\xspace}
\newcommand{\om}{\textsc{OM}\xspace}
\newcommand{\jump}{\textsc{Jump}\xspace}
\newcommand{\N}{{\mathbb N}}
\newcommand{\R}{{\mathbb R}}
\newcommand{\eps}{\varepsilon}

\DeclareMathOperator{\Bin}{Bin}
\DeclareMathOperator{\pow}{pow}
\DeclareMathOperator*{\argmax}{arg\,max}

\newtheorem{theorem}{Theorem}
\newtheorem{lemma}[theorem]{Lemma}

\begin{document}

\title{Runtime Analysis of a Heavy-Tailed $(1+(\lambda,\lambda))$ Genetic Algorithm on Jump Functions}

\author{Denis Antipov \\
		ITMO University \\
  		St. Petersburg, Russia \\
  		and \\
		Laboratoire d'Informatique (LIX), \\
		CNRS, \'Ecole Polytechnique, \\ 
		Institut Polytechnique de Paris \\
		Palaiseau, France \\
		\and
		Benjamin Doerr \\
		Laboratoire d'Informatique (LIX), \\
		CNRS, \'Ecole Polytechnique, \\ 
		Institut Polytechnique de Paris \\
		Palaiseau, France \\
} 

\maketitle
{\sloppy

\begin{abstract}
    It was recently observed that the $(1+(\lambda,\lambda))$ genetic algorithm can comparably easily escape the local optimum of the jump functions benchmark. Consequently, this algorithm can optimize the jump function with jump size $k$ in an expected runtime of only $n^{(k + 1)/2}k^{-k/2}e^{O(k)}$ fitness evaluations (Antipov, Doerr, Karavaev (GECCO 2020)). To obtain this performance, however, a non-standard parameter setting depending on the jump size $k$ was used.
    
    To overcome this difficulty, we propose to choose two parameters of the $(1+(\lambda,\lambda))$ genetic algorithm randomly from a power-law distribution. Via a mathematical runtime analysis, we show that this algorithm with natural instance-independent choices of the distribution parameters on all jump functions with jump size at most $n/4$ has a performance close to what the best instance-specific parameters in the previous work obtained. This price for instance-independence can be made as small as an $O(n\log(n))$ factor. Given the difficulty of the jump problem and the runtime losses from using mildly suboptimal fixed parameters (also discussed in this work), this appears to be a fair price.
\end{abstract}

\section{Introduction}
\label{sec:intro}

The $(1 + (\lambda,\lambda))$ genetic algorithm (\ollea) is a still fairly simple evolutionary algorithm proposed at GECCO 2013 (journal version~\cite{DoerrDE15}). Through a combination of mutation with a high mutation rate and crossover with the parent as repair mechanism, it tries to increase the speed of exploration without compromising in terms of exploitation. The mathematical analyses on \onemax~\cite{DoerrDE15,DoerrD18} and easy random satisfiability instances~\cite{BuzdalovD17} showed that the new algorithm has a moderate advantage over classic evolutionary algorithms (EAs). Some experimental results~\cite{GoldmanP14,MironovichB17} also suggested that this algorithm is promising.

More recently, a mathematical analysis on jump functions showed that here the \ollea with the right parameter setting outperforms the classic algorithms by a much wider margin than on the simpler problems regarded before. One drawback of this result is that the choice of the parameters is non-trivial. In particular, (i)~one needed to deviate from the previous recommendation to connect the mutation rate $p$ and the crossover bias $c$ to the population size $\lambda$ via $p = \frac{\lambda}{n}$ and $c = \frac 1 {\lambda}$, and (ii)~the optimal parameters depended heavily on the difficulty parameter $k$ of the jump functions class. While also many sub-optimal parameter values gave an improvement over classic algorithms, the non-trivial influence of the parameters on the algorithm performance still raises the question if one can (at least partially) relieve the algorithm designer from the task of choosing the parameters. 

In this work, we make a big step forward in this direction. We deduce from previous works that taking mutation rate $p$ and crossover bias $c$ equal can be a good idea when making progress is difficult (these parameters were found suitable in the last stages of the \onemax optimization process and to cross the fitness valley of jump functions). Parameterizing $p = c = \sqrt{s/n}$, we obtain that an offspring after mutation and crossover has an expected Hamming distance of $s$ from the parent. Hence the parameter $s$, in a similar manner as the mutation rate in a traditional mutation-based algorithm, quantifies the typical search radius of the \ollea. With this (heuristic) reduction of the parameter space, it remains to choose suitable values for this search radius and for the offspring population size $\lambda$. 

The last years have seen a decent number of self-adjusting or self-adapting parameter choices (e.g.,~\cite{LassigS11,MambriniS15,DangL16ppsn,DoerrDY16ppsn,DoerrDK18,DoerrGWY19,DoerrWY18}, see also the survey~\cite{DoerrD20bookchapter}) including a self-adjusting choice of $\lambda$ for the \ollea optimizing \onemax~\cite{DoerrDE15,DoerrD18} and easy random SAT instance~\cite{BuzdalovD17}. In all these successful applications of dynamic parameter settings, the characteristic of the optimization process changes only slowly over time, which enables the algorithm to adjust to the changing environment. We are therefore not too optimistic that these ideas work well on problems like jump functions, which show a sudden change from easy \onemax-style optimization to a difficult-to-cross fitness valley. 

For this reason, we preferred a random choice of the parameters. The work~\cite{DoerrLMN17} has demonstrated that a random choice (from a heavy-tailed distribution) of the mutation rate for the \oea optimizing jump functions can give very good results. Hence trying this idea for our parameter $s$ is very natural. There is less a-priori evidence that a random choice of the value for $\lambda$ is a good idea, but we have tried this nevertheless. We note that the recent work~\cite{AntipovBD20} showed that the \ollea with a heavy-tailed choice of $\lambda$ and the previous recommendation $p = \frac \lambda n$ and $c = \frac 1 \lambda$ has a good performance on \onemax, but it is not clear to us why this should indicate also a good performance on jump functions, in particular, with our different choice of $p$ and $c$.

We conduct a mathematical runtime analysis of the \ollea with heavy-tailed choices of $s$ and $\lambda$ (the heavy-tailed \ollea for short) from a broad range of power-law distributions. It shows that for a power-law exponent $\beta_s > 1$ for the choice of $s$ and a power-law exponent $\beta_\lambda$ equal to two or slightly above, a very good performance can be obtained. The resulting runtimes are slightly higher than those stemming from the best, instance-specific static parameters, but still much below the runtimes of classic evolutionary algorithms.

While undoubtedly we have obtained parameters that work uniformly well over all jump functions, we also feel that our choices of the power-law exponent are quite natural, so that the name parameterless \ollea might be justified. There is not much to say on the choice of $s$, where apparently all power-laws (with exponent greater than one, which is a very natural assumption for any use of a power-law) give good results. For the choice of $\lambda$, we note that the cost of one iteration of the \ollea is $2\lambda$ fitness evaluations. Hence $2E[\lambda]$ is the cost of an iteration with a random choice of $\lambda$. Now any power-law exponent $\beta_\lambda > 2$ gives a constant value for $E[\lambda]$. The larger $\beta_\lambda$ is, the more the power-law distribution is concentrated on constant values. For constant $\lambda$, however, the \ollea cannot profit a lot from the intermediate selection step, and thus shows a behavior similar to classic mutation-based algorithms. For this reason, choosing a power-law exponent rather close to two appears to be a natural choice. Based both on this informal argument and our mathematical results, for a practical application of our algorithm we recommend to use $\beta_s$ slightly above one, say $1.1$, and $\beta_\lambda$ slightly above two, say $2.1$.

The asymptotically best choice of $\beta_\lambda$ (in the sense that the worst-case price for being instance-independent is lowest) is obtained from taking $\beta_\lambda = 2$. Since this alone would give an infinite value for $E[\lambda]$
, one needs to restrict the range of values this distribution is defined on. To obtain an $O(n k^{\beta_s -1})$ price of instance-independence, a generous upper bound of $2^n$ is sufficient. To obtain our best price of instance-independence of $O(n \log n)$, a similar trick is necessary for the choice of $s$, namely taking $\beta_s = 1$ and capping the range at the (trivial) upper bound $s \le n$. While we think that these considerations are interesting from the theoretical perspective as they explore the limits of our approach, we do not expect these hyperparameter choices to be useful in many practical applications. We note the runtime of the \oea with heavy-tailed mutation rate was shown~\cite{DoerrLMN17} to exceed the instance-specific best runtime of the \oea by a factor of $\Theta(n^{\beta-0.5})$. Hence a power-law exponent $\beta$ as low as possible (but larger than one) looks best from the theoretical perspective. In contrast, in the experiments in~\cite{DoerrLMN17}, no improvement was seen from lowering $\beta$ below $1.5$.

The remainder of this paper is structured as follows. In the following preliminaries section, we introduce the jump functions benchmark and the heavy-tailed \ollea along with some relevant previous works. Section~\ref{sec:dynamic_params} contains the heart of this work, a mathematical runtime analysis of the heavy-tailed \ollea on jump functions. In Section~\ref{sec:static_params}, we show via an elementary computational analysis that the \ollea with fixed parameters is very sensitive to missing the optimal parameter values. This suggests that the small (polynomial) price of our one-size-fits-all solution is well invested compared to the performance losses stemming from missing the optimal static parameter values.

\section{Preliminaries}
\label{sec:preliminaries}

In this section we collect all necessary definitions and tools, which we use in the paper. We only use standard notation such as the following. By $\N$ we denote the set of positive integers. We write $[a..b]$ to denote an integer interval including its borders and $(a..b)$ to denote an integer interval excluding its borders. For $a, b \in \R$ the notion $[a..b]$ means $[\lceil a \rceil..\lfloor b \rfloor]$. For the real-valued intervals we write $[a, b]$ and $(a, b)$ respectively.
For any probability distribution $\mathcal{L}$ and random variable $X$, we write $X\sim\mathcal{L}$ to indicate that $X$ follows the law $\mathcal{L}$.
We denote the binomial law with parameters $n \in \N$ and $p \in [0,1]$ by $\Bin\left(n, p\right)$.

\subsection{\jump Functions} 

The family of jump functions is a class of model functions based on the classic \onemax benchmark function. \onemax is a pseudo-Boolean function defined on the space of bit-strings of length $n$, which returns the number of one-bits in its argument. More formally,
\begin{align*}
    \onemax(x) = \om(x) = \sum_{i = 1}^n x_i.
\end{align*}

The $\jump_k$ function with jump size $k$ is then defined as follows.

\begin{align*}
    \jump_k(x) = 
    \begin{cases}
        \om(x) + k, \text{ if } \om(x) \in [0..n - k] \cup \{n\}, \\
        n - \om(x), \text{ if } \om(x) \in [n - k + 1..n - 1].    
    \end{cases}
\end{align*}

A plot of $\jump_k$ is shown in Figure~\ref{fig:jump}. Different from \onemax, this function has a fitness valley which is hard to cross for the many EAs. For example, the \mplea and \mclea for all values of $\mu$ and $\lambda$ need an expected time of $\Omega(n^k)$ to optimize $\jump_k$~\cite{DrosteJW02,Doerr20}. With a heavy-tailed mutation operator, the runtime of the \oea can be lowered by a $k^{\Theta(k)}$ factor, so it remains $\Theta(n^k)$ for $k$ constant. Better runtimes have been shown for algorithms using crossover and other mechanisms, see~\cite{JansenW02,FriedrichKKNNS16,DangFKKLOSS16,DangFKKLOSS18,RoweA19,WhitleyVHM18}, though in our view only the $O(n^{k-1})$ runtime in~\cite{DangFKKLOSS18} stems from a classic algorithm with natural parameters. 

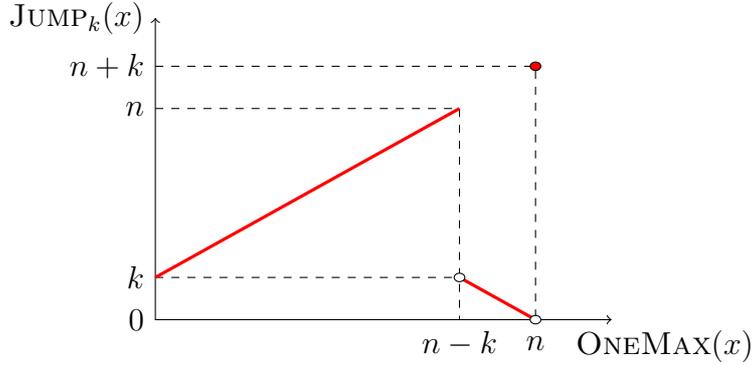
\begin{figure}
    \begin{center}
     \begin{tikzpicture}[yscale = 0.8]
  
      \draw [dashed] (0, 4.2) -- (5, 4.2) -- (5, 0);
      \draw [dashed] (0, 3.5) -- (4, 3.5) -- (4, 0);
      \draw [dashed] (0, 0.7) -- (4, 0.7);
  
      \draw [{<[scale=1.5]}-{>[scale=1.5]}] (0, 5) -- (0, 0) -- (6, 0);
      \draw [very thick, red] (0, 0.7) -- (4, 3.5);
      \draw [very thick, red] (4, 0.7) -- (5, 0);

      \draw [fill=white] (4, 0.7) circle (0.7mm);
      \draw [fill=white] (5, 0) circle (0.7mm);
      \draw [fill=red, draw=black] (5, 4.2) circle (0.7mm);
  
      \node [below] at (6.7, 0) {$\onemax(x)$};
      \node [below] at (5, -0.1) {$n$};
      \node [left] at (0, 0) {$0$};
      \node [below] at (4, 0) {$n - k$};
      \node [left] at (0, 0.7) {$k$};
      \node [left] at (0, 3.5) {$n$};
      \node [left] at (0, 4.2) {$n + k$};
      \node [left] at (0, 5) {$\jump_k(x)$};
  
     \end{tikzpicture}
    \end{center}
    \caption{Plot of the $\jump_k$ function. As a function of unitation, the function value of a search point $x$ depends only on the number of one-bits in~$x$.}
    \label{fig:jump}
  \end{figure}

\subsection{Power-Law Distribution} 

We say that a random variable $X \in \N$ follows a power-law distribution with parameters $\beta$ and $u$ if 
\begin{align*}
    \Pr[X = i] = 
    \begin{cases}
        C_{\beta, u} i^{-\beta}, \text{ if } i \in [1..u], \\
        0, \text{ else,}
    \end{cases}
\end{align*}
where $C_{\beta, u} = (\sum_{j = 1}^u j^{-\beta})^{-1}$ is the normalization coefficient. We write $X \sim \pow(\beta, u)$ and call $u$ the upper limit of $X$ and $\beta$ the power-law exponent. We note that if $\beta > 1$, then $\Pr[X = i] = \Theta(1)$ for any integer $i = \Theta(1)$. At the same time the distribution is \emph{heavy-tailed}, which means that we have a decent (only inverse polynomial instead of negative-exponential) probability that $X = i$ for any super-constant $i \le u$. If $\beta > 2$, then we also have $E[X] = \Theta(1)$. These properties are easily seen from the following estimates of the partial sums of the generalized harmonic series, which we will frequently need in this work.

\begin{lemma}
    \label{lem:sum_estimates}
    For all positive integers $a$ and $b$ such that $b \ge a$ and for all $\beta > 0$, the sum $\sum_{i = a}^b i^{-\beta}$ is
    \begin{itemize}
        \item $\Theta((b + 1)^{1 - \beta} - a^{1 - \beta})$, if $\beta \in [0, 1)$,
        \item $\Theta(\log(\frac{b + 1}{a}))$, if $\beta = 1$, and
        \item $\Theta(a^{1 - \beta} - (b + 1)^{1 - \beta})$, if $\beta > 1$.
    \end{itemize}
\end{lemma}

This lemma is easily shown by approximating the sums via integrals. 
It gives the following estimates for the normalization coefficient $C_{\beta, u}$ of the power-law distribution and for the expected value of $X\sim \pow(\beta, u)$.

\begin{lemma}
    \label{lem:c}
    The normalization coefficient $C_{\beta, u} = (\sum_{j = 1}^u i^{-\beta})^{-1}$ of the power-law distribution with parameters $\beta$ and $u$ is
    \begin{itemize}
        \item $\Theta(u^{\beta - 1})$, if $\beta \in [0, 1)$,
        \item $\Theta(1/\log(u + 1))$, if $\beta = 1$, and 
        \item $\Theta(1)$, if $\beta > 1$.
    \end{itemize} 
\end{lemma}

\begin{lemma}
    \label{lem:expectation}
    The expected value of $X\sim \pow(\beta, u)$ is
    \begin{itemize}
        \item $\Theta(u)$, if $\beta \le 1$,
        \item $\Theta(u^{2 - \beta})$, if $\beta \in (1, 2)$,
        \item $\Theta(\log(u + 1))$, if $\beta = 2$, and 
        \item $\Theta(1)$, if $\beta > 2$.
    \end{itemize}   
\end{lemma}

\subsection{The Heavy-Tailed \ollea}

We now define the heavy-tailed \ollea as motivated in the introduction. The main difference from the standard \ollea is that at the start of each iteration the mutation rate $p$, the crossover bias $c$, and the population sizes $\lambda_m$ and $\lambda_c$ for the mutation and crossover phases are randomly chosen as follows. We sample $s \sim \pow(\beta_s, u_s)$ and take $p = c = (\frac{s}{n})^{1/2}$. 
The population sizes are chosen via $\lambda_m = \lambda_c = \lambda \sim \pow(\beta_\lambda, u_\lambda)$. Here the upper limits $u_\lambda$ and $u_s$ can be any positive integers and the power-law exponents $\beta_\lambda$ and $\beta_s$ can be any non-negative real numbers. We call these parameters of the power-law distribution the \emph{hyperparameters of the heavy-tailed \ollea} and we give recommendations on how to choose these hyperparameters in Section~\ref{sec:esc_loc_opt}. The pseudocode of this algorithm is shown in Algorithm~\ref{alg:pseudo}. We note that it is not necessary to store the whole offspring populations, since only the best individual has a chance to be selected as mutation or crossover winner. Hence also large values for $\lambda$ are algorithmically feasible.

\begin{algorithm}[h]
    $x \gets $ random bit string of length $n$\;
    \While{not terminated}
        {
        Choose $s \sim \pow(\beta_s, u_s)$\;
        $p \gets (\frac{s}{n})^{1/2}$\;
        $c \gets (\frac{s}{n})^{1/2}$\;
        Choose $\lambda \sim \pow(\beta_\lambda, u_\lambda)$\;
        \textbf{Mutation phase:}\\
        Choose $\ell \sim \Bin\left(n, p\right)$\;
        \For{$i \in [1..\lambda]$}
            {$x^{(i)} \gets$ a copy of $x$\;
            Flip $\ell$ bits in $x^{(i)}$ chosen uniformly at random\;
            }
        $x' \gets \argmax_{z \in \{x^{(1)}, \dots, x^{(\lambda)}\}}f(z)$\;
        \textbf{Crossover phase:}\\
        \For{$i \in [1..\lambda]$}
            {Create $y^{(i)}$ by taking each bit from $x'$ with probability $c$ and from $x$ with probability $(1 - c)$\;
            }
        $y \gets \argmax_{z \in \{y^{(1)}, \dots, y^{(\lambda)}\} }f(z)$\;
        \If{$f(y) \ge f(x)$}
            {
             $x \gets y$\;   
            }
        }
    \caption{The heavy-tailed \ollea maximizing a pseudo-Boolean function~$f$.}
    \label{alg:pseudo}
\end{algorithm}

The few existing results for the \ollea with static parameters show the following: With optimal static parameters, the algorithm optimizes \onemax in time roughly $O(n \sqrt{\log n})$~\cite{DoerrD18}. With a suitable fitness dependent parameter choice or a self-adjusting parameter choice building on the one-fifth rule, this runtime can be lowered to $O(n)$. Due to the weaker fitness-distance correlation, only slightly inferior results have been shown in~\cite{BuzdalovD17} for sufficiently dense random satisfiability instances in the planted solution model (and the experiments in~\cite{BuzdalovD17} suggest that indeed the algorithm suffers from the weaker fitness-distance correlation). A runtime analysis~\cite{AntipovDK19} on \leadingones gave no better runtimes than the classic $O(n^2)$ bound, but at least it showed that also in the absence of a good fitness-distance correlation the \ollea can be efficient by falling back to the optimization behavior of the \oea.

We use the following language (also for the standard \ollea with fixed values for $p, c, \lambda_m, \lambda_c$). We denote by $T_I$ and $T_f$ the number of iterations and the number of fitness evaluations performed until some event holds (which is always specified in the text). If the algorithm has already reached the local optimum, then we call the mutation phase \emph{successful} if all $k$ zero-bits of $x$ are flipped to ones in the mutation winner~$x'$. We also call an offspring of the mutation phase \emph{good} if it has all $k$ zero-bits flipped. If the algorithm has not reached the local optimum, then we call the mutation phase \emph{successful} if $x'$ contains a one-bit not present in $x$. In this case we call an offspring \emph{good} if it has at least one zero-bit flipped to one and does not lie in the fitness valley of $\jump_k$. We call the crossover phase \emph{successful} if the crossover winner has a greater fitness than $x$. The \emph{good} offspring in the crossover phase is the one which has a better fitness than $x$.

To estimate the probability of a true progress in one iteration we use the following lemma, which can easily be deduced from Lemmas~3.1 and~3.2 in~\cite{AntipovDK20}.

\begin{lemma}
    \label{lem:successful_iter}
    Let $\lambda_m = \lambda_c = \lambda$ and $p = c = (\frac{s}{n})^{1/2}$ with $s \in [k..2k]$. If the current individual $x$ of the \ollea is in the local optimum of $\jump_k$, then the probability that the algorithm finds the global optimum in one iteration is at least $e^{-\Theta(k)}\min\{1, (\frac{k}{n})^k \lambda^2 \}$.
\end{lemma}

\subsection{Wald's Equation}
\label{sec:tools}

%
Since not only the number of iterations until the optimum is found is a random variable, but also the number of fitness evaluations in each iteration, we shall use the following version of Wald's equation~\cite{Wald45} to estimate the number of fitness evaluations until the optimum is found.

\begin{lemma}
    \label{lem:wald}
    Let $(X_t)_{t \in \N}$ be a sequence of non-negative real-valued random variables with identical finite expectation. Let $T$ be a positive integer random variable with finite expectation. If for all $i \in \N$ event $(T \ge i)$ is independent of $(X_t)_{t = i}^{+\infty}$, then
	\[
		E\left[\sum_{t = 1}^{T} X_t\right] = E[T]E[X_1].	
	\]
\end{lemma}

\section{Heavy-Tailed Parameters: Runtime Analysis}
\label{sec:dynamic_params}

In this section we conduct a rigorous runtime analysis for our heavy-tailed \ollea optimizing jump functions with jump size $k \in [2..\frac{n}{4}]$. We cover the full spectrum of the algorithm's hyperparameters $\beta_s, u_s, \beta_\lambda, u_\lambda$. For large ranges of the hyperparameters, in particular, for natural values like $\beta_s = \beta_\lambda = 2+\eps$ and $u_s = u_\lambda = \infty$, we observe a performance that is only a little worse than the one with the best instance-specific static parameters. This price of instance-independence can be brought down to an $O(n\log(n))$ factor. Taking into account the effect of failing to guess the optimal parameters shown in Section~\ref{sec:static_params}, this is a fair price for a one-size-fits-all algorithm.

Since a typical optimization process on jump functions consists of two very different regimes, we analyze separately the difficult regime of going from the local optimum to the global one (Section~\ref{sec:esc_loc_opt}) and the easy \onemax-style regime encountered before that (Section~\ref{sec:reach_loc_opt}).

\subsection{Escaping the Local Optimum}
\label{sec:esc_loc_opt}

The time to leave the local optimum (necessarily to the global one) is described in the following theorem and Table~\ref{tbl:runtime}. We will see later that unless $\beta_\lambda < 2$, and this is not among our recommended choices, or $k = 2$, the time to reach the local optimum is not larger than the time to go from the local to the global optimum. Hence for $\beta_\lambda \ge 2$, the table also gives valid runtime estimates for the complete runtime. 


\begin{theorem}
    \label{thm:runtime}
    Let $k \in [2..\frac n4]$. Assume that we use the heavy-tailed \ollea (Algorithm~\ref{alg:pseudo}) to optimize $\jump_k$, starting already in the local optimum. Then the expected number of the fitness evaluations until the optimum is found is shown in Table~\ref{tbl:runtime}, where $p_s$ denotes the probability that $s \in [k .. 2k]$. If $u_s \ge 2k$, then $p_s$ is
    \begin{itemize}
        \item $\Theta((\frac{k}{u_s})^{1 - \beta_s})$, if $\beta_s \in [0, 1)$,
        \item $\Theta(\frac{1}{\ln(u_s)})$, if $\beta_s = 1$, and
        \item $\Theta(k^{\beta_s - 1})$, if $\beta_s > 1$.
    \end{itemize} 
\end{theorem}

\begin{table}[t]
    \caption{Influence of the four hyperparameters $\beta_s, u_s, \beta_\lambda, u_\lambda$ on the expected number $E[T_f]$ of fitness evaluations the heavy-tailed \ollea starting in the local optimum takes to optimize $\jump_k$. Since all runtime bounds are of type $E[T_f] = F(\beta_\lambda,u_\lambda) / p_s$, where $p_s = \Pr[s \in [k..2k]]$, to ease reading we only state $F(\beta_\lambda,u_\lambda) = E[T_f] p_s$. By taking $\beta_s = 2+\eps$ or $\beta_s = 2 \wedge u_s = n$, one obtains $p_s = k^\eps$ or $p_s = O(\log n)$. Using $\beta_\lambda = 2$ and an exponential $u_\lambda$ gives the lowest price of an $O(n \log n)$ factor for being independent of the instance parameter $k$. We also advertise the slightly inferior combination $\beta_\lambda = 2 + \eps$ and $u_\lambda = +\infty$ as for $\beta_\lambda > 2$ each iteration has a constant expected cost and $u_\lambda$ has no influence on the runtime (if chosen large enough). If $\beta_\lambda \ge 2$ and $k \ge 3$, then the times stated are also the complete runtimes starting from a random initial solution.}
	\label{tbl:runtime}
	\begin{center}
		\begin{tabular}{|c||c|c|}
			\hline
            $\beta_\lambda$ & 
            $E[T_f]p_s$ if $u_\lambda < \left(\frac{n}{k}\right)^{k/2}$ &  
            $E[T_f]p_s$ if $u_\lambda \ge \left(\frac{n}{k}\right)^{k/2}$ \\ \hline

            $[0, 1)$ & 
            \multirow{3}{*}{$e^{\Theta(k)} \frac{1}{u_\lambda}\left(\frac{n}{k}\right)^k$}  & 
            $u_\lambda e^{\Theta(k)}$ \\ \cline{1-1}\cline{3-3}
            
            $=1$ &  & 
            $u_\lambda e^{\Theta(k)}/\left(1 + \ln\left(u_\lambda\left(\frac{n}{k}\right)^{k/2}\right)\right)$ \\ \cline{1-1}\cline{3-3}

            $(1, 2)$ &  &  
            $e^{\Theta(k)}u_\lambda^{2 - \beta} \left(\frac{n}{k}\right)^{k/2(\beta - 1)}$ \\ \hline

            $=2$ & 
            $e^{\Theta(k)} \frac{\ln(u_\lambda + 1)}{u_\lambda} \left(\frac{n}{k}\right)^k$ & 
            \cellcolor{green!20!lightgray!60} $e^{\Theta(k)} \ln(u_\lambda) \left(\frac{n}{k}\right)^{k/2}$ \\ \hline

            $(2, 3)$ & 
            $e^{\Theta(k)} \frac{1}{u_\lambda^{3 - \beta}} \left(\frac{n}{k}\right)^k$ & 
            \cellcolor{green!20!lightgray!60} $e^{\Theta(k)} \left(\frac{n}{k}\right)^{k/2(\beta - 1)}$ \\ \hline
            
            $=3$ & 
            $e^{\Theta(k)} \frac{1}{\ln(u_\lambda + 1)} \left(\frac{n}{k}\right)^k$ & 
            $e^{\Theta(k)} \left(\frac{n}{k}\right)^k / \ln\left(\left(\frac{n}{k}\right)^k\right)$ \\ \hline

            $>3$ & 
            \multicolumn{2}{c|}{$e^{\Theta(k)} \left(\frac{n}{k}\right)^k$} \\ \hline
		\end{tabular}
	\end{center}
\end{table}

Before the proof we distill the following recommendations on how to set the parameters of the power-law distributions from Theorem~\ref{thm:runtime}.

\emph{Distribution of $\lambda$:} We note that when guessing $u_\lambda$ right (depending on $k$), and only then, then good runtimes can be obtained for $\beta_\lambda < 2$. Since we aim at a (mostly) parameterless approach, this is not very interesting. When $\beta_\lambda > 3$, we observe a slow runtime behavior similar to the one of the \oea with heavy-tailed mutation rate~\cite{DoerrLMN17}. This is not surprising since with this distribution of $\lambda$ typically only small values of $\lambda$ are sampled. We profit most from the strength of the heavy-tailed \ollea when $\beta_\lambda$ is close to two. If it is larger than two, then each iteration has an expected constant cost, so we can conveniently choose $u_\lambda = \infty$ without that this can have a negative effect on the runtime. This is a hyperparameter setting we would recommend as a first, low-risk attempt to use this algorithm. Slightly better results are obtained from using $\beta_\lambda=2$. Now a finite value for $u_\lambda$ is necessary, but the logarithmic influence of $u_\lambda$ on the runtime allows to be generous, e.g., taking $u_\lambda$ exponential in $n$. Smaller values lead to minimally better runtimes as long as one stays above the boundary $(\frac nk)^{k/2}$, so optimizing here is risky. 

\emph{Distribution of $s$:} The distribution of $s$ is less critical as long as $u_s \ge 2k$. Aiming at an algorithm free from critical parameter choices, we therefore recommend to take $u_s = n$ unless there is a clear indication that only short moves in the search space are necessary. Once we decided on $u_s = n$, a $\beta_s$ value below one is not interesting (apart from very particular situations). Depending on what jump sizes we expect to encounter, taking $\beta_s = 1$ leading to an $O(\log n)$-factor contribution of $s$ to the runtime or taking $\beta_s = 1+\eps$, $\eps > 0$ but small, leading to an $O(k^{\eps})$-factor contribution to the runtime are both reasonable choices.

\begin{proof}[Proof of Theorem~\ref{thm:runtime}]
    Let $F$ be the event that the algorithm finds the global optimum in one iteration when the current individual $x$ is already in the local optimum. The probability $P$ of this event is at least
    \begin{align*}
        P \ge p_{(F \mid s)} p_s,
    \end{align*}
    where $p_{(F \mid s)} = \Pr[F \mid s \in [k..2k]]$ and $p_s = \Pr[s \in [k..2k]]$. The expected number of iterations $T_I$ until we find the optimum is therefore
    \begin{align*}
        E[T_I] = \frac{1}{P} \le \frac{1}{p_{(F \mid s)} p_s}.
    \end{align*}
    In each iteration the heavy-tailed \ollea performs $2\lambda$ fitness evaluation (where $\lambda$ is chosen from a power-law distribution at the start of the iteration). Using Wald's equation (Lemma~\ref{lem:wald}) we compute the expected runtime $T_f$ in terms of fitness evaluations from $T_I$. 
    \begin{align*}
        E[T_f] = E[T_I]E[2\lambda] = \frac{E[2\lambda]}{P} \le \frac{2E[\lambda]}{ p_{(F \mid s)} p_s}.
    \end{align*}
    In the remainder of the proof we estimate how $E[\lambda]$, $p_{(F \mid s)}$, and $p_s$ depend on the hyperparameters of the algorithm.
    
    The expected value of $\lambda$ is
    \begin{align*}
        E[\lambda] = \sum_{i = 1}^{u_\lambda} p_\lambda(i) i = C_{\beta_\lambda, u_\lambda} \sum_{i = 1}^{u_\lambda} i^{1 - \beta_\lambda},
    \end{align*}
    where $p_\lambda(i) = \Pr[\lambda = i]$. We compute the conditional probability of $F$ as 
    \begin{align*}
        p_{(F \mid s)} = \sum_{i = 1}^{u_\lambda} p_\lambda(i)p_{(F \mid s,\lambda)}(i), 
    \end{align*}
    where $p_{(F \mid s,\lambda)}(i) = \Pr[F \mid s \in [k..2k] \wedge \lambda = i]$. Note that event $\lambda = i$ does not depend on the choice of $s$. By Lemma~\ref{lem:successful_iter} we have
    \begin{align*}
        p_{(F \mid s,\lambda)}(i) \ge \begin{cases}
            \left(\frac{k}{n}\right)^k i^2 e^{-\Theta(k)}, \text{ if } i \le \left(\frac{n}{k}\right)^{k/2},\\
            e^{-\Theta(k)}, \text{ else.}
        \end{cases}
    \end{align*}
    We consider two cases of the size of $u_\lambda$ relative to $k$ and $n$. First, if $u_\lambda < (\frac{n}{k})^{k/2}$, then we have 
    \begin{align*}
        p_{(F \mid s)} &\ge \sum_{i = 1}^{u_\lambda} C_{\beta_\lambda, u_\lambda} i^{-\beta_\lambda} \left(\frac{k}{n}\right)^k i^2 e^{-\Theta(k)} =  C_{\beta_\lambda, u_\lambda} e^{-\Theta(k)} \left(\frac{k}{n}\right)^k \sum_{i = 1}^{u_\lambda} i^{2 - \beta_\lambda}.
    \end{align*}
    Hence, we have
    \begin{align*}
        E[T_f] = \frac{C_{\beta_\lambda, u_\lambda} \sum_{i = 1}^{u_\lambda} i^{1 - \beta_\lambda}}{p_s C_{\beta_\lambda, u_\lambda} e^{-\Theta(k)} \left(\frac{k}{n}\right)^k \sum_{i = 1}^{u_\lambda} i^{2 - \beta_\lambda}} = p_s^{-1} e^{\Theta(k)} \left(\frac{n}{k}\right)^k \frac{\sum_{i = 1}^{u_\lambda} i^{1 - \beta_\lambda}}{\sum_{i = 1}^{u_\lambda} i^{2 - \beta_\lambda}}.
    \end{align*}

    In the second case, if $u \ge (\frac{n}{k})^{k/2}$, we have
    \begin{align*}
            p_{(F \mid s)} &\ge \sum_{i = 1}^{\lfloor \left(\frac{n}{k}\right)^{k/2} \rfloor} C_{\beta_\lambda, u_\lambda} i^{-\beta_\lambda} \left(\frac{k}{n}\right)^k i^2 e^{-\Theta(k)} +   \sum_{\lfloor \left(\frac{n}{k}\right)^{k/2} \rfloor + 1}^{u_\lambda} C_{\beta_\lambda, u_\lambda} i^{-\beta_\lambda} e^{-\Theta(k)} \\
                            &=   C_{\beta_\lambda, u_\lambda} e^{-\Theta(k)} \left(\left(\frac{k}{n}\right)^k \sum_{i = 1}^{\lfloor \left(\frac{n}{k}\right)^{k/2} \rfloor} i^{2 - \beta_\lambda} + \sum_{i = \lfloor \left(\frac{n}{k}\right)^{k/2} \rfloor + 1}^{u_\lambda} i^{-\beta_\lambda}\right).
    \end{align*}
    
    Therefore,
    \begin{align*}
        E[T_f] &\le \frac{2E[\lambda]}{p_{(F \mid s)} p_s} \le \frac{e^{\Theta(k)}\sum_{i = 1}^{u_\lambda} i^{1 - \beta_\lambda}}{p_s\left(\left(\frac{k}{n}\right)^k \sum_{i = 1}^{\lfloor \left(\frac{n}{k}\right)^{k/2} \rfloor} i^{2 - \beta_\lambda} + \sum_{i =\lfloor \left(\frac{n}{k}\right)^{k/2} \rfloor + 1}^{u_\lambda} i^{-\beta_\lambda}\right)}.
    \end{align*}
    
    Viewing these two cases together,  we obtain
    \begin{align}\label{eq:T_f}
        E[T_f] \le \frac{e^{\Theta(k)}S_1}{p_s \left(\left(\frac{k}{n}\right)^k S_2 + S_0\right)},
    \end{align}
    where 
    \begin{itemize}
        \item $S_1 \coloneqq \sum_{i = 1}^{u_\lambda} i^{1 - \beta_\lambda}$,
        \item $S_2 \coloneqq \sum_{i = 1}^{\min\{\lfloor (\frac{n}{k})^{k/2}, u_\lambda \rfloor} i^{2 - \beta_\lambda}$, and
        \item $S_0 \coloneqq \sum_{i =\lfloor (\frac{n}{k})^{k/2} \rfloor + 1}^{u_\lambda} i^{-\beta_\lambda}$ if $u_\lambda > (\frac{n}{k})^{k/2}$ and $S_0 \coloneqq 0$ otherwise.
    \end{itemize}
    
    Table~\ref{tbl:terms} shows the estimates of $S_1$, $S_2$ and $S_0$, which follow from Lemma~\ref{lem:sum_estimates}. We also note that the estimates for $p_s = C_{\beta_s, u_s} \sum_{i = k}^{2k} i^{-\beta}$ follow from Lemmas~\ref{lem:sum_estimates} and~\ref{lem:c}. We omit these elementary calculations. Putting these estimates into~\eqref{eq:T_f} proves the theorem.
\end{proof}
    
\begin{table}
    \caption{The values of $S_1$, $S_2$ and $S_0$ used in the proof of Theorem~\ref{thm:runtime}.}
    \label{tbl:terms}
    \begin{center}
        \begin{tabular}{|c||c|c|c|}
            \hline
            $\beta_\lambda$ & 
            $S_1$ &  
            $S_2$ & 
            $S_0$ if $u_\lambda > (\frac{n}{k})^{k/2}$ \\ \hline

            $[0, 1)$ & 
            \multirow{3}{*}{$\Theta(u_\lambda^{2 - \beta_\lambda})$}  & 
            \multirow{5}{*}{$\Theta\left((\min\{u_\lambda, (\frac{n}{k})^{k/2} \})^{3 - \beta_\lambda}\right)$} & 
            $\Theta\left(u_\lambda^{1 - \beta_\lambda} - (\frac{n}{k})^{k(1 - \beta_\lambda)/2}\right)$ \\ \cline{1-1}\cline{4-4}

            $=1$ &  &  & 
            $\Theta\left(\ln\left(u_\lambda(\frac{k}{n})^{k/2}\right) \right)$ \\ \cline{1-1}\cline{4-4}

            $(1, 2)$ &  &  & 
            \multirow{5}{*}{$\Theta\left((\frac{n}{k})^{k(1 - \beta_\lambda)/2} - u_\lambda^{1 - \beta_\lambda}\right)$} \\ \cline{1-2}

            $=2$ & 
            $\Theta(\log(u_\lambda))$ &  &  \\ \cline{1-2}
            $(2, 3)$ & 
            \multirow{3}{*}{$\Theta(1)$} &  &  \\ \cline{1-1}\cline{3-3}
            
            $=3$ &  & 
            $\Theta\left(\log(\min\{u_\lambda, (\frac{n}{k})^{k/2} \})\right)$ &   \\ \cline{1-1}\cline{3-3}

            $>3$ &  & 
            $\Theta(1)$ &   \\ \hline
        \end{tabular}
    \end{center}
\end{table}

\subsection{Reaching the Local Optimum}
\label{sec:reach_loc_opt}

In this section we show that the heavy-tailed choice of parameters lets the \ollea reach the local optimum relatively fast. Without proof, we note that if $\beta_\lambda \ge 2$ and $k \ge 3$, then the time to reach the local optimum is not larger than the time to go from the local to the global optimum. For a set of hyperparameters giving the best price for instance-independence, we now show an $O(n^2\log^2(n))$ time bound for reaching the local optimum. 

\begin{theorem}
    \label{thm:runtime_onemax}
    Let $u_\lambda = 2^{\Theta(n)}$, $\beta_\lambda = 2$, $u_s = \Theta(n)$, and $\beta_s = 1$. Then the expected runtime until the heavy-tailed \ollea reaches the local optimum of $\jump_k$ starting in a random string is at most $O(n^2\log^2(n))$ fitness evaluations. For greater $\beta_\lambda$ and any $u_\lambda$ this runtime is at most $O(n \log^2(n))$. In both cases with $\beta_s > 1$ and any $u_s \in \N$ the runtime is reduced by a $\Theta(\log(n))$ factor.
\end{theorem}

\begin{proof}
    We prove the theorem only for $\beta_\lambda = 2$ and $\beta_s = 1$, since for other hyperparameter values the arguments are identical.
    By Lemma~\ref{lem:c}, the probability $p_{s,\lambda}$ to choose $s = 1$ and $\lambda = 1$ is 
    \begin{align*}
        p_{s, \lambda} = C_{\beta_s, u_s} 1^{(-1)} C_{\beta_\lambda u_\lambda} 1^{(-2)} = \Theta\left(\frac{1}{\log(n)}\right).
    \end{align*}

    With $s = 1$ and $\lambda = 1$ the algorithm essentially performs an iteration of the \oea with mutation rate $\frac{1}{n}$, since there is no selection of the mutation winner and each bit of the crossover offspring is flipped with probability $\sqrt{\frac{1}{n}}^2 = \frac{1}{n}$. Therefore, if the algorithm has not reached the local optimum, then the probability $P$ to have a true progress in one iteration is at least
    \begin{align*}
        P \ge p_{s,\lambda} \frac{n - i}{n},
    \end{align*}
    where $i$ is the current fitness of $x$. Therefore, by Lemma~\ref{lem:sum_estimates} the expected number of iterations until the algorithm reaches the local optimum is at most
    \begin{align*}
        E[T_I] \le \sum_{i = 0}^{n - k - 1} \frac{n}{p_{s,\lambda}(n - i)} \le \Theta(\log(n)) \cdot n \cdot O(\log(n)) = O(n\log^2(n)). 
    \end{align*}

    Since by Lemma~\ref{lem:expectation} the expected number of fitness evaluations per iteration is $\Theta(\log(u_\lambda)) = \Theta(n)$, by Wald's equation (Lemma~\ref{lem:wald}) we have
    \begin{align*}
        E[T_f] &= E[T_I]E[2\lambda] \le O(n\log^2(n)) \cdot \Theta(n) = O(n^2\log^2(n)).      
    \end{align*}
\end{proof}

\section{Static Parameters}
\label{sec:static_params}

In~\cite{AntipovDK20} it was shown that the \ollea can solve $\jump_k$ in $(\frac{n}{k})^{k/2}e^{O(k)}$ fitness evaluations when it starts in the local optimum. This is, if we ignore the $e^{O(k)}$ factor, the square root of the runtime of the best mutation-based algorithms~\cite{DoerrLMN17}. However, such an upper bound can be obtained only by setting the parameters of the algorithm to values which depend on the jump size $k$. In this section we show that a deviation from these instance-specific optimal parameters setting significantly increases the runtime. The consequence is that when the parameter $k$ is unknown, we are not likely to choose a good static parameter setting.

To analyze the negative effect of a wrong parameter choice we use the precise expression of the probability $P$ to go from the local to the global optimum in one iteration, which is 
\begin{align}\label{eq:p_static}
    P = \sum_{\ell = 0}^n p_\ell p_m(\ell) p_c(\ell),
\end{align}
where $p_\ell$ is the probability to choose $\ell$ bits to flip, $p_m(\ell)$ is the probability of a successful mutation phase conditional on the chosen $\ell$, and $p_c(\ell)$ is the probability of a successful crossover phase conditional on the chosen $\ell$ and on the mutation being successful.

Since $\ell \sim \Bin(n, p)$, we have $p_\ell = \binom{n}{\ell}p^\ell(1 - p)^{n - \ell}$. The probability of a successful mutation depends on the chosen $\ell$. If $\ell < k$, then it is impossible to flip all $k$ zero-bits, hence $p_m(\ell) = 0$. For larger $\ell$ the probability to create a good offspring in a single application of the mutation operator is $q_m(\ell) = \binom{n - k}{\ell - k} / \binom{n}{\ell}$. If $\ell \in [k + 1..2k - 1]$ then any good offspring occurs in the fitness valley and has a worse fitness than any other offspring that is not good. Hence, in order to have a successful mutation we need all $\lambda_m$ offspring to be good. Therefore, the probability of a successful mutation is $(q_m(\ell))^{\lambda_m}$. For $\ell = k$ and $\ell \ge 2k$ we are guaranteed to choose a good offspring as the winner of the mutation phase if there is at least one. Therefore, the mutation phase is successful with probability $p_m(\ell) = 1 - (1 - q_m(\ell))^{\lambda_m}$.

In the crossover phase we can create a good offspring only if $\ell \ge k$. For this we need to take all $k$ bits which are zero in $x$ from $x'$, and then take all $\ell - k$ one-bits which were flipped from $x$. The probability to do so in one offspring is $q_c(\ell) = c^k (1 - c)^{\ell - k}$. Since we create $\lambda_c$ offspring and at least one of them must be superior to $x$, the probability of the successful crossover phase is $p_c(\ell) = 1 - (1 - c^k(1 - c)^{\ell - k})^{\lambda_c}$.

Putting these probabilities into~\eqref{eq:p_static} we obtain
\begin{align*}
    P &= \binom{n}{k} p^k (1 - p)^{n - k} \left(1 - \left(1 - \binom{n}{k}^{-1}\right)^{\lambda_m}\right) \left(1 - (1 - c^k)^{\lambda_c}\right) \\
    &+ \sum_{\ell = k + 1}^{2k - 1} \binom{n}{\ell} p^\ell (1 - p)^{n - \ell} \left(\frac{\binom{n - k}{\ell - k}}{\binom{n}{\ell}}\right)^{\lambda_m} \left(1 - (1 - c^k(1 - c)^{\ell - k})^{\lambda_c}\right) \\
    &+ \sum_{\ell = 2k}^n \binom{n}{\ell} p^\ell (1 - p)^{n - \ell} \left(1 - \left(1 - \frac{\binom{n - k}{\ell - k}}{\binom{n}{\ell}}\right)^{\lambda_m}\right) \left(1 - (1 - c^k(1 - c)^{\ell - k})^{\lambda_c}\right). \\
\end{align*}

Via this expression for $P$ we compute the expected runtime in terms of iterations as $E[T_I] = P^{-1}$ and the expected runtime in terms of fitness evaluations as $E[T_f] = (\lambda_m + \lambda_c)P^{-1}$.
It is hard estimate precisely the probability $P$ and thus the expected runtime. Therefore, to show the critical influence of the parameters on the runtime, we compute $E[T_f]$ precisely for $n = 2^{20}$ and $k \in \{2^2, 2^4, 2^6\}$ and for different parameter values. We fix $\lambda_m = \lambda_c = \sqrt{\frac nk}^k$ and take $p = 2^\delta \sqrt{\frac kn}$ and $c = 2^{-\delta} \sqrt{\frac kn}$ for all $\delta \in [-\log_2(\sqrt{\frac nk})..\log_2(\sqrt{\frac nk})]$; these limits for $\delta$ guarantee that both $p$ and $c$ do not exceed $1$. Note that we preserve the invariant $pcn = k$, since otherwise the expected Hamming distance between $x$ and any crossover offspring (the ``search radius'' of the \ollea) is not $k$, which makes it even harder to find the global optimum. These values (for $\delta = 0$) were suggested in~\cite{AntipovDK20} (based on an asymptotic analysis, so constant factors were ignored). The results of this computation are shown in Figure~\ref{fig:deviation}.

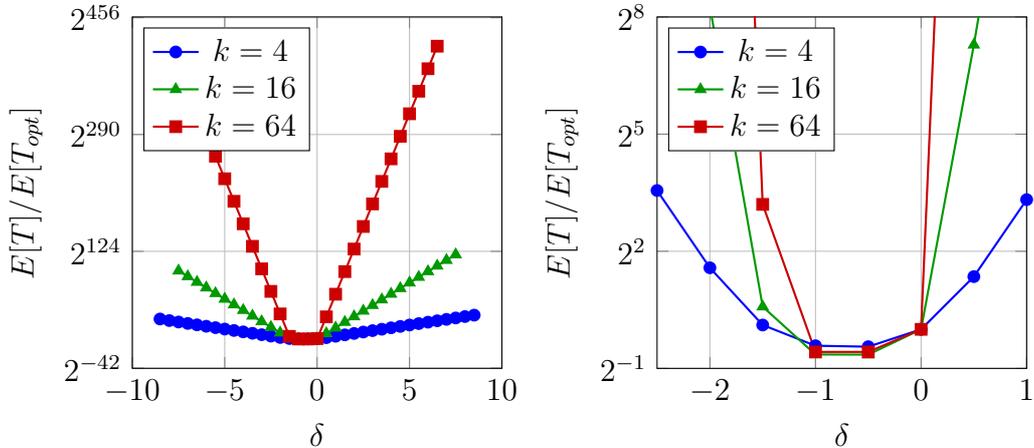
\begin{figure}[t]
    \begin{tikzpicture}
        \begin{axis}[width=0.47\linewidth, height=0.3\textheight, ymode=log, log base y=2, grid=major, xmin = -10, xmax = 10, xlabel={$\delta$}, ylabel={$E[T] / E[T_{opt}]$}, legend pos=north west, cycle list name=myplotcycle, every axis plot/.append style={thick}]
            \addplot plot [] coordinates {(-8.5, 2.409e+08) (-8.0, 5.148e+07) (-7.5, 1.238e+07) (-7.0, 3.083e+06) (-6.5, 7.706e+05) (-6.0, 1.926e+05) (-5.5, 4.816e+04) (-5.0, 1.204e+04) (-4.5, 3.010e+03) (-4.0, 7.525e+02) (-3.5, 1.881e+02) (-3.0, 4.704e+01) (-2.5, 1.176e+01) (-2.0, 2.983e+00) (-1.5, 1.080e+00) (-1.0, 7.476e-01) (-0.5, 7.351e-01) (0.0, 1.000e+00) (0.5, 2.550e+00) (1.0, 9.998e+00) (1.5, 4.002e+01) (2.0, 1.602e+02) (2.5, 6.411e+02) (3.0, 2.565e+03) (3.5, 1.026e+04) (4.0, 4.107e+04) (4.5, 1.643e+05) (5.0, 6.572e+05) (5.5, 2.629e+06) (6.0, 1.052e+07) (6.5, 4.207e+07) (7.0, 1.683e+08) (7.5, 6.731e+08) (8.0, 2.693e+09) (8.5, 1.077e+10)};
            \addlegendentry{ $k = 4$};
            \addplot plot [] coordinates {(-7.5, 1.014e+29) (-7.0, 3.726e+26) (-6.5, 1.454e+24) (-6.0, 5.681e+21) (-5.5, 2.219e+19) (-5.0, 8.669e+16) (-4.5, 3.386e+14) (-4.0, 1.323e+12) (-3.5, 5.167e+09) (-3.0, 2.018e+07) (-2.5, 7.884e+04) (-2.0, 3.080e+02) (-1.5, 1.496e+00) (-1.0, 6.398e-01) (-0.5, 6.384e-01) (0.0, 1.000e+00) (0.5, 1.561e+02) (1.0, 4.048e+04) (1.5, 1.046e+07) (2.0, 2.695e+09) (2.5, 6.930e+11) (3.0, 1.780e+14) (3.5, 4.567e+16) (4.0, 1.171e+19) (4.5, 3.001e+21) (5.0, 7.689e+23) (5.5, 1.969e+26) (6.0, 5.044e+28) (6.5, 1.292e+31) (7.0, 3.307e+33) (7.5, 8.468e+35)};
            \addlegendentry{ $k = 16$};
            \addplot plot [] coordinates {(-6.5, 1.824e+97) (-6.0, 4.242e+87) (-5.5, 9.876e+77) (-5.0, 2.299e+68) (-4.5, 5.354e+58) (-4.0, 1.246e+49) (-3.5, 2.902e+39) (-3.0, 6.757e+29) (-2.5, 1.573e+20) (-2.0, 3.663e+10) (-1.5, 9.193e+00) (-1.0, 6.699e-01) (-0.5, 6.699e-01) (0.0, 1.000e+00) (0.5, 2.026e+09) (1.0, 9.657e+18) (1.5, 4.464e+28) (2.0, 2.019e+38) (2.5, 8.997e+47) (3.0, 3.966e+57) (3.5, 1.735e+67) (4.0, 7.548e+76) (4.5, 3.271e+86) (5.0, 1.414e+96) (5.5, 6.102e+105) (6.0, 2.629e+115) (6.5, 1.132e+125)};
            \addlegendentry{ $k = 64$};
        \end{axis}
    \end{tikzpicture}
    \begin{tikzpicture}
        \begin{axis}[width=0.47\linewidth, height=0.3\textheight, ymode=log, log base y=2, grid=both, xmin = -2.5, xmax = 1, ymin = 0.5, ymax = 256, xlabel={$\delta$}, ylabel={$E[T] / E[T_{opt}]$}, legend pos=north west, cycle list name=myplotcycle, every axis plot/.append style={thick}]
            \addplot plot [] coordinates {(-2.5, 1.176e+01) (-2.0, 2.983e+00) (-1.5, 1.080e+00) (-1.0, 7.476e-01) (-0.5, 7.351e-01) (0.0, 1.000e+00) (0.5, 2.550e+00) (1.0, 9.998e+00)};
            \addlegendentry{ $k = 4$};
            \addplot plot [] coordinates {(-2.5, 7.884e+04) (-2.0, 3.080e+02) (-1.5, 1.496e+00) (-1.0, 6.398e-01) (-0.5, 6.384e-01) (0.0, 1.000e+00) (0.5, 1.561e+02) (1.0, 4.048e+04)};
            \addlegendentry{ $k = 16$};
            \addplot plot [] coordinates {(-2.5, 1.573e+20) (-2.0, 3.663e+10) (-1.5, 9.193e+00) (-1.0, 6.699e-01) (-0.5, 6.699e-01) (0.0, 1.000e+00) (0.5, 2.026e+09) (1.0, 9.657e+18)};
            \addlegendentry{ $k = 64$}
        \end{axis}
    \end{tikzpicture}
    \caption{The ratio of the runtime with disturbed parameters to the runtime with the parameters suggested in~\cite{AntipovDK20}. The left plot shows the full picture for all considered values of $\delta$. The right plot shows in more detail a smaller interval around the best values.}
    \label{fig:deviation}
\end{figure}

As one can see, there is a relatively small interval where losses in runtime are of a small constant factor (for $\delta = -1$ the runtime is even slightly better), but generally the runtime is increased by a $\Theta(2^{|\delta|k})$ factor. Therefore, in order to solve $\jump_k$ effectively with the \ollea using the static parameters, one has to guess the value of $k$ with a small relative error. In practice when we optimize some \jump-like problem we usually cannot tell in advance the size of jump which we must perform to escape local optima. Therefore, the general recommendation is to prefer the choice of the parameters from a power-law distribution to (well-tuned) static parameters.  

\section{Conclusion}

In this work, we proposed a variant of the \ollea with a heavy-tailed choice of both the population size $\lambda$ and the search radius $s$. To the best of our knowledge, this is the first time that two parameters of an EA are chosen in this manner. Our mathematical runtime analysis showed that this algorithm with suitable, but natural choices of the distribution parameters can optimize all jump functions in a time that is only mildly higher than the runtime of the \ollea with the best known instance-specific parameter values.

We are optimistic that the insights gained on the jump functions benchmark extend, at least to some degree, also to other non-unimodal problems. Clearly, supporting this hope with rigorous results is an interesting direction for future research. From a broader perspective, this work suggests to try to use heavy-tailed parameter choices for more than one parameter simultaneously. Our rigorous results indicate that the prices for ignorant (heavy-tailed) choices of parameters simply multiply. For a small number of parameters with critical influence on the performance, this might be a good deal.

\newcommand{\etalchar}[1]{$^{#1}$}

}

\end{document}